\documentclass{article}

\usepackage[preprint]{icml2026}
\usepackage{microtype}
\usepackage{graphicx}
\usepackage{subcaption}
\usepackage{booktabs}
\usepackage{amsmath,amssymb,amsthm}
\usepackage{mathtools}
\usepackage{algorithm}
\usepackage{algorithmic}
\usepackage{enumitem}
\usepackage{hyperref}
\usepackage{natbib}

\newcommand{\E}{\mathbb{E}}
\newcommand{\Pcal}{\mathcal{P}}
\newcommand{\Xcal}{\mathcal{X}}
\newcommand{\Ycal}{\mathcal{Y}}
\newcommand{\Zcal}{\mathcal{Z}}
\newcommand{\Hcal}{\mathcal{H}}

\newcommand{\TV}{\mathrm{TV}}

\newcommand{\MB}{\mathrm{MB}}
\newcommand{\Pa}{\mathrm{Pa}}

\newcommand{\argmin}{\operatorname*{arg\,min}}

\theoremstyle{definition}
\newtheorem{definition}{Definition}[section]
\newtheorem{assumption}{Assumption}[section]
\newtheorem{remark}{Remark}[section]

\theoremstyle{plain}
\newtheorem{lemma}{Lemma}[section]
\newtheorem{proposition}{Proposition}[section]
\newtheorem{theorem}{Theorem}[section]

\AtBeginDocument{%
  \setlength{\abovedisplayskip}{5pt plus 2pt minus 2pt}      % Default ~10pt
  \setlength{\belowdisplayskip}{5pt plus 2pt minus 2pt}      % Default ~10pt
  \setlength{\abovedisplayshortskip}{3pt plus 1pt minus 1pt} % Default ~0pt+3pt
  \setlength{\belowdisplayshortskip}{4pt plus 2pt minus 1pt} % Default ~6pt+3pt
}

\makeatletter
\renewcommand{\section}{\@startsection{section}{1}{\z@}{-0.10in}{0.01in}%
   {\normalfont\large\bf\raggedright}}
\renewcommand{\subsection}{\@startsection{subsection}{2}{\z@}{-0.08in}{0.005in}%
   {\normalfont\normalsize\bf\raggedright}}
\renewcommand{\paragraph}{\@startsection{paragraph}{4}{\z@}{0.8ex plus 0.2ex minus 0.2ex}{-1em}%
   {\normalfont\normalsize\bf}}
\makeatother

\icmltitlerunning{PRISM: Structure-Aware DP Synthetic Data for Prediction}

\begin{document}

\twocolumn[
\icmltitle{PRISM: Differentially Private Synthetic Data with Structure-Aware Budget Allocation for Prediction}

\begin{icmlauthorlist}
\icmlauthor{Amir Asiaee}{vumc-biostat}
\icmlauthor{Chao Yan}{vumc-bmi}
\icmlauthor{Zachary B.\ Abrams}{wustl}
\icmlauthor{Bradley A.\ Malin}{vumc-bmi}
\end{icmlauthorlist}

\icmlaffiliation{vumc-biostat}{Department of Biostatistics, Vanderbilt University Medical Center, 2525 West End Avenue, Nashville, TN 37203, USA}
\icmlaffiliation{vumc-bmi}{Department of Biomedical Informatics, Vanderbilt University Medical Center, Nashville, TN, USA}
\icmlaffiliation{wustl}{Institute for Informatics, Washington University, 4444 Forest Park Avenue, St.\ Louis, MO 63108, USA}

\icmlcorrespondingauthor{Amir Asiaee}{amir.asiaeetaheri@vumc.org}

\icmlkeywords{Differential Privacy, Synthetic Data, Predictive Modeling, Causal Inference, Graphical Models, Workload Design}

\vskip 0.3in
]
\printAffiliationsAndNotice{}

\begin{abstract}
Differential privacy (DP) provides a mathematical guarantee limiting what an adversary can learn about any individual from released data. However, achieving this protection typically requires adding noise, and noise can accumulate when many statistics are measured. Existing DP synthetic data methods treat all features symmetrically, spreading noise uniformly even when the data will serve a specific prediction task. This uniform treatment is suboptimal: if a hospital releases synthetic patient records for readmission prediction, perturbing all features equally---rather than concentrating fidelity on those most relevant to readmission---fails to optimize utility for the prediction task.

To address this, we develop a prediction-centric approach operating in three regimes depending on available structural knowledge. In the \emph{causal regime}, when the causal parents of $Y$ (the outcome variable to predict) are known and distribution shift is expected, we target the parents for robustness. In the \emph{graphical regime}, when a Bayesian network structure is available and the distribution is stable, the Markov blanket of $Y$ provides a sufficient feature set for optimal prediction. In the \emph{predictive regime}, when no structural knowledge exists, we select features via differentially private methods without claiming to recover causal or graphical structure.

We formalize this as PRISM (Prediction-centric Release with Informed Structure Measurements), a mechanism that (i) identifies a predictive feature subset according to the appropriate regime, (ii) constructs targeted summary statistics involving those features and $Y$, (iii) allocates budget to minimize an upper bound on prediction error, and (iv) synthesizes data via graphical-model inference. We prove end-to-end privacy guarantees and risk bounds connecting measurement noise to predictive performance. Empirically, task-aware allocation improves prediction accuracy compared to generic synthesizers, with larger gains under tighter privacy budgets. Under distribution shift, targeting causal parents achieves AUC $\approx 0.73$ while correlation-based selection collapses to chance ($\approx 0.49$).
\end{abstract}

% ============================================================
\section{Introduction}
\label{sec:intro}

Consider a hospital network that wants to share patient records with external researchers developing readmission risk models. Privacy regulations and institutional policies often make it challenging to release raw data directly, so the hospital generates synthetic records designed to preserve statistical properties of the original while protecting individuals. The researchers have one goal, to train classifiers predicting 30-day readmission. Yet standard synthetic data generators treat all 200 columns identically, spending equal effort to preserve correlations among demographic fields, billing codes, and clinical indicators. For readmission prediction, perhaps only 15 features matter. Preserving their relationship to the outcome would suffice; accurately reproducing correlations among irrelevant features is wasted effort.

\paragraph{Why budget matters.}
Differential privacy (DP) formalizes privacy protection by requiring that any computation on the data be insensitive to whether any single individual is included \citep{dwork2006calibrating,dwork2014algorithmic}. In practice, this means adding carefully calibrated noise. The total amount of noise a mechanism can add while maintaining a given privacy level is often called the ``privacy budget.'' Crucially, when a synthesizer measures many statistics---say, all pairwise feature correlations---the noise accumulates across measurements. This creates a fundamental trade-off: measure more statistics and each becomes noisier, or measure fewer statistics more accurately.

Current DP synthesizers, whether based on low-dimensional marginals \citep{zhang2017privbayes,mckenna2019pgm,mckenna2021nist,zhang2021privsyn} or deep generative models \citep{jordon2019pategan,fang2022dpctgan}, are predominantly task-agnostic. They approximate the full joint distribution or a generic collection of statistics, treating all features as equally important. When the released data serves a known prediction task, this symmetry squanders budget on structure irrelevant to prediction.

\paragraph{A prediction-centric alternative.}
Rather than spreading noise uniformly, we concentrate the privacy budget on statistical relationships that determine predictive performance for a designated target $Y$. Features strongly influencing $Y$ are measured accurately; features with weak or no predictive value are measured coarsely or omitted. The result is synthetic data optimized for a specific downstream task, with formal DP guarantees.

\paragraph{Identifying the right features.}
Which features deserve careful measurement? A natural answer, ``the most predictive ones,'' can be misleading. A feature strongly correlated with $Y$ in training data may be a spurious correlate that fails under distribution shift. The answer depends on what structural knowledge is available and what robustness properties are desired.

We distinguish three regimes. In the \emph{causal regime}, domain knowledge identifies the causal parents $\Pa(Y)$, the direct causes of $Y$ in a structural causal model. Using parents provides robustness under distribution shift because the causal mechanism $P(Y \mid \Pa(Y))$ tends to remain stable even when marginal distributions or spurious correlations change \citep{pearl2009causality}. In the \emph{graphical regime}, a Bayesian network structure is available (from domain knowledge or prior studies), and the \emph{Markov blanket} $\MB(Y)$, comprising parents, children, and co-parents, provides a sufficient set for Bayes-optimal prediction in a fixed distribution \citep{koller2009pgm}. This is a probabilistic sufficiency, not a causal robustness guarantee. In the \emph{predictive regime}, no structural knowledge exists, so we select features via differentially private methods without claiming to recover any particular structure \citep{thakurta2013stability}.

These regimes form a hierarchy of assumptions: the causal regime requires the strongest knowledge but provides robustness under shift; the graphical regime requires structure but not causal semantics; the predictive regime is the most practical but offers no structural guarantees.

\paragraph{An end-to-end perspective.}
Synthetic data is never an end in itself---downstream algorithms consume it for learning. Formalizing this as a pipeline,
\[
D \xrightarrow{\text{DP synthesizer } \mathsf{M}} \widetilde{D} \xrightarrow{\text{learner } \mathsf{A}} \widehat{h},
\]
we can trace how mechanism design choices propagate to the prediction error of $\widehat{h}$. This enables principled optimization: allocate budget to minimize an upper bound on excess risk, rather than a generic fidelity metric.

\paragraph{What distinguishes our approach.}
Our contribution is not the synthesis machinery (we build on Private-PGM \citep{mckenna2019pgm}), but the integration of a three-regime framework; task-derived workload construction; structure-guided feature targeting; risk-motivated budget allocation with a closed-form solution; and a DP feature-selection variant with end-to-end accounting.

\paragraph{Contributions.}
We formalize prediction-centric DP synthesis across the causal, graphical, and predictive regimes; distinguish fixed-distribution prediction sufficiency (Markov blanket) from shift-robust targeting (causal parents); develop PRISM by combining regime-appropriate feature targeting, task-aware workload construction, closed-form budget allocation, and PGM-based synthesis; and validate the regime view empirically on SCM shifts and Adult.

\paragraph{Scope.}
Our goal is neither DP feature selection alone nor direct model release via DP-SGD \citep{abadi2016dpsgd}. We produce a reusable synthetic dataset that downstream analysts can query, visualize, and model freely---but one whose fidelity is concentrated on task-relevant structure. This fills a gap between fully task-agnostic synthesis and single-query model release.

% ============================================================
\section{Related Work}
\label{sec:related}
DP synthetic data for tabular data is dominated by marginal-based methods that measure low-dimensional statistics and fit a distribution consistent with noisy marginals (e.g., PrivBayes \citep{zhang2017privbayes}, Private-PGM \citep{mckenna2019pgm}, MST \citep{mckenna2021nist}, PrivSyn \citep{zhang2021privsyn}) and by DP-trained deep generative models (e.g., PATE-GAN \citep{jordon2019pategan}, DP-CTGAN \citep{fang2022dpctgan}). These approaches are largely task-agnostic: they target generic fidelity rather than utility for a designated prediction target.

Workload-aware query release and synthesis methods (e.g., MWEM \citep{hardt2012mwem}, DualQuery \citep{gaboardi2014dualquery}, AIM \citep{mckenna2022aim}) improve accuracy for specified workloads but take the workload as input. Task-specific private release methods \citep{su2015privpfc,vietri2022rappp} are more targeted but are coupled to particular model families or multi-task settings. PRISM differs by deriving a task workload from a single prediction target $Y$ and allocating privacy budget to minimize a prediction-relevant risk bound, while leveraging standard PGM post-processing as a backend.

% ============================================================
\section{Preliminaries}
\label{sec:prelim}

\subsection{Data model and notation}
Let $Z = (X,Y) \in \Zcal = \Xcal \times \Ycal$, where $X = (X_1,\dots,X_d)$ is a feature vector and $Y$ is a designated prediction target. The private dataset is
\[
D = \{z_i\}_{i=1}^n = \{(x_i,y_i)\}_{i=1}^n \in \Zcal^n,
\]
assumed drawn i.i.d.\ from an unknown population distribution $P \in \Pcal(\Zcal)$ unless otherwise stated.

Our main results are cleanest for finite/discretized domains (typical in marginal-based synthesizers). When $X$ contains continuous columns, we assume public binning/clipping or use mixed-type query families; see Section~\ref{sec:experiments} and remarks.

\subsection{Differential privacy}
A \emph{query} is a function $q: \Zcal^n \to \mathcal{O}$ that computes a statistic from a dataset, such as a histogram, mean, or marginal distribution. A \emph{mechanism} $\mathsf{M}: \Zcal^n \to \mathcal{O}$ is a randomized algorithm that answers queries while protecting individual privacy. Differential privacy formalizes this protection.

We use the standard \emph{replace-one} neighboring relation: datasets $D, D'$ are neighbors (written $D \sim D'$) if they have the same size and differ in exactly one record.

\begin{definition}[Differential Privacy \citep{dwork2006calibrating,dwork2014algorithmic}]
A mechanism $\mathsf{M}:\Zcal^n \to \mathcal{O}$ is $(\varepsilon,\delta)$-differentially private if for all neighboring datasets $D \sim D'$ and all measurable $S \subseteq \mathcal{O}$,
\[
\Pr[\mathsf{M}(D)\in S] \le e^\varepsilon \Pr[\mathsf{M}(D')\in S] + \delta.
\]
\end{definition}

The \emph{sensitivity} of a query measures how much its output can change when one record is replaced:
\[
\Delta_q := \max_{D \sim D'} \|q(D) - q(D')\|.
\]
For $\ell_2$ sensitivity the norm is Euclidean; for $\ell_1$ sensitivity, the sum of absolute differences. DP mechanisms add noise calibrated to sensitivity: lower sensitivity permits less noise for the same privacy guarantee.

We use the Gaussian mechanism for numeric query release: for a query $q$ with $\ell_2$ sensitivity $\Delta_q$, release $q(D) + \eta$ where $\eta \sim \mathcal{N}(0, \sigma^2 I)$ and $\sigma = \Delta_q \sqrt{2\log(1.25/\delta)}/\varepsilon$ \citep{dwork2014algorithmic}. For discrete selection (e.g., feature selection), we use the exponential mechanism: sample $c\in\mathcal{C}$ with probability proportional to $\exp(\varepsilon \cdot u(D,c) / (2\Delta_u))$ \citep{mcsherry2007mechanism}.

\subsection{Causal graphs and Markov blanket}
Let $G$ be a directed acyclic graph (DAG) over variables $(X_1,\dots,X_d,Y)$. Denote the parents of $Y$ by $\Pa(Y)$. The Markov blanket $\MB(Y)$ is the minimal set of nodes that d-separates $Y$ from all other nodes; in a Bayesian network it consists of $Y$'s parents, children, and co-parents of its children \citep{koller2009pgm}. Under standard conditions (causal Markov + faithfulness), $Y \perp\!\!\!\perp X_{\setminus \MB(Y)} \mid X_{\MB(Y)}$.

% ============================================================
\section{Problem Formulation}
\label{sec:problem}

We now formalize prediction-centric DP synthetic data release, starting from a pipeline objective and then developing tractable surrogates.

\subsection{A pipeline objective: DP synthesis for downstream learning}
Let $\mathsf{M}$ be a DP synthesizer mapping private data $D$ to synthetic data $\widetilde{D}$, and let $\mathsf{A}$ be a learning algorithm producing a predictor $\widehat{h}$ from a hypothesis class $\Hcal$.

For a bounded loss $\ell:\mathcal{Y}\times\mathcal{Y}\to[0,1]$ (e.g., 0-1 classification error), define population risk under $P$:
\[
\mathcal{R}_P(h) := \E_{(X,Y)\sim P}\big[\ell(h(X),Y)\big].
\]

\begin{definition}[Prediction-centric DP synthesis objective]
Given target $Y$ and hypothesis class $\Hcal$, design a mechanism $\mathsf{M}$ that is $(\varepsilon,\delta)$-DP and minimizes the expected pipeline risk
\[
\mathcal{U}(\mathsf{M};\mathsf{A},P)
:= \E_{D\sim P^n}\E_{\widetilde{D}\sim \mathsf{M}(D)} \big[\mathcal{R}_P(\mathsf{A}(\widetilde{D}))\big].
\]
\end{definition}

This objective captures what we care about---downstream prediction accuracy---but is not directly tractable: $P$ is unknown and downstream analysts may use various learning algorithms. We therefore develop distributional surrogates that connect to measurable quantities.

\subsection{A distributional surrogate: preserve the right marginal(s)}
For predictors that use only a subset $S\subseteq[d]$ of features, risk depends only on the joint distribution of $(X_S,Y)$. This suggests a tractable target: ensure synthetic distribution $\widetilde{P}$ matches $P$ well on $(X_S,Y)$ while tolerating distortion elsewhere.

\begin{definition}[Task marginal distance]
For subset $S\subseteq[d]$, define
\[
\Delta_S(P,Q) := \TV(P_{S,Y},Q_{S,Y}),
\]
where $\TV(\cdot,\cdot)$ denotes total variation distance.
\end{definition}

\paragraph{Why total variation?}
We use TV because it provides a clean, distribution-free bound on risk differences: for any bounded loss $\ell\in[0,1]$ and any predictor $h$ depending only on $X_S$, the risk difference $|\mathcal{R}_P(h) - \mathcal{R}_Q(h)|$ is at most $\TV(P_{S,Y},Q_{S,Y})$ (Lemma~\ref{lem:tv-risk}). This bound holds without assumptions on the form of $P$ or $h$, making it suitable for settings where downstream analysts choose their own models. Other distances (KL divergence, Wasserstein) would require additional assumptions or yield looser bounds.

\paragraph{Predictor assumption.}
Our risk bounds apply when the downstream predictor depends only on features in $S$: $h(x) = h(x_S)$. This is satisfied when (i) the analyst explicitly restricts to $S$, (ii) regularization or feature selection drives coefficients on $X_{\setminus S}$ to zero, or (iii) the population relationship satisfies $Y \perp X_{\setminus S} \mid X_S$ so that optimal predictors ignore $X_{\setminus S}$. Case (iii) is guaranteed when $S = \MB(Y)$.

\subsection{Choosing $S$: three regimes}
The feature subset $S$ can be specified in three ways, forming a hierarchy of assumptions with different trade-offs:

\begin{assumption}[Markov property]
\label{def:markov}
The joint distribution $P$ over $(X,Y)$ is Markov with respect to DAG $G$: each variable is conditionally independent of its non-descendants given its parents.
\end{assumption}

\paragraph{Regime 1: Causal (parents known, shift expected).}
When a structural causal model identifies the parents $\Pa(Y)$ and distribution shift is anticipated between training and deployment, set $S = \Pa(Y)$. The causal mechanism $P(Y \mid \Pa(Y))$ tends to remain stable under shifts that change marginal distributions or break spurious correlations. This regime provides the strongest robustness guarantees but requires causal knowledge, which typically comes from domain expertise rather than data. No privacy cost is incurred for selecting $S$.

\paragraph{Regime 2: Graphical (structure known, fixed distribution).}
When a Bayesian network structure is available (from domain knowledge, prior studies, or learned from public data) but no shift is expected, set $S = \MB(Y)$. Under the Markov assumption (Definition~\ref{def:markov}), this choice is sufficient for Bayes-optimal prediction (Proposition~\ref{prop:mb-sufficiency}). The blanket may include children and co-parents of $Y$, which are predictive in the training distribution but may be less robust under shift. When the structure must be learned from private data via conditional independence tests, part of the privacy budget is consumed. This regime provides prediction optimality without requiring causal semantics.

\paragraph{Regime 3: Predictive (no structure, agnostic).}
When no graphical or causal structure is available, $S$ must be selected from data using differentially private feature selection (exponential mechanism or stability-based methods \citep{thakurta2013stability}). This consumes part of the privacy budget. The selected $S$ is a predictive subset; we do not claim it recovers the Markov blanket or causal parents. It may include spurious correlates, offering less robustness under shift, but still focuses budget on features empirically associated with $Y$. This is the most practical regime when structural knowledge is unavailable.

\subsection{Design levers: workload and budget allocation}
Given $S$, our mechanism measures a collection of statistics (the workload $\mathcal{W}$) under DP and fits a distribution consistent with noisy answers. Two design choices determine downstream utility:
workload design (which marginals to include, especially those involving $Y$ and $X_S$) and budget allocation (how to divide privacy budget so task-critical marginals receive more budget and less noise).

In the resulting synthetic data, relationships involving $S$ and $Y$ are preserved more accurately, while other relationships may be distorted. This is the desired trade-off: prioritize prediction-relevant structure.

% ============================================================
\section{PRISM: A Structure-Aware DP Synthesis Mechanism}
\label{sec:method}

This section describes PRISM (Prediction-centric Release with Informed Structure Measurements), a concrete mechanism implementing the prediction-centric approach across all three regimes. We specify each step in sufficient detail for reproduction.
Algorithm~\ref{alg:prism} in Supplementary Theory (Appendix~\ref{sec:supp-theory}) provides a compact pseudocode summary.

\subsection{Overview}
\textbf{Input:} Private dataset $D=\{(x_i,y_i)\}_{i=1}^n$ with $d$ features, privacy parameters $(\varepsilon,\delta)$, designated target column $Y$, optional causal graph $G$, and hyperparameters: subset size $k$ (if $G$ unavailable), maximum marginal order $r$ (default 3), and synthetic sample size $n_{\mathrm{syn}}$.

\textbf{Output:} Synthetic dataset $\widetilde{D}$ of size $n_{\mathrm{syn}}$ with identical schema, satisfying $(\varepsilon,\delta)$-DP, with fidelity concentrated on predicting $Y$.

\subsection{Step 1: selecting predictive features $S$}

\paragraph{Regime 1: Causal (parents known).}
When domain expertise identifies the causal parents $\Pa(Y)$ and distribution shift is expected, set $S \leftarrow \Pa(Y)$. This yields the smallest feature set and provides robustness under shift. No privacy cost is incurred.

\paragraph{Regime 2: Graphical (structure known).}
When a DAG $G$ is available but no shift is expected, set $S \leftarrow \MB(Y)$ for prediction-optimal synthesis. The blanket includes parents, children, and co-parents. If the structure must be learned from private data, use DP conditional independence tests, consuming part of the privacy budget. No privacy cost if the structure comes from domain knowledge or public sources.

\paragraph{Regime 3: Predictive (data-driven selection).}
Without structural knowledge, privately select $S$ using one of two methods:

\textbf{(3a) Greedy exponential selection.}
Define a score $u(D,j)$ measuring association between feature $j$ and $Y$. We use a clipped $\chi^2$ statistic: let $\chi^2_j$ be the chi-squared statistic for the $(X_j, Y)$ contingency table, and set $u(D,j) = \min(\chi^2_j / n, B)$ for a clipping bound $B$ (we use $B=1$). This score has sensitivity $O(1/n)$ under replace-one neighbors. Select $k$ features greedily: in round $t$, use the exponential mechanism with budget $\varepsilon_{\mathrm{sel}}/k$ to select feature $j_t$ from remaining candidates proportional to $\exp(\varepsilon_{\mathrm{sel}} \cdot u(D,j) / (2k \cdot \Delta_u))$, where $\Delta_u$ is the sensitivity.

\textbf{(3b) Stability-based selection.}
Fit a DP-regularized logistic regression (e.g., via objective perturbation or DP-SGD) and take $S$ as the support of the fitted coefficients. This approach is attractive for large $d$ and provides theoretical robustness guarantees \citep{thakurta2013stability}.

Both methods compose with subsequent measurement via standard DP composition. The selected $S$ is a predictive subset; we do not claim it recovers the Markov blanket or causal parents.

\subsection{Step 2: constructing the workload $\mathcal{W}$}
The workload $\mathcal{W} = \mathcal{W}_{\mathrm{task}} \cup \mathcal{W}_{\mathrm{bg}}$ comprises two parts:

\paragraph{Task workload $\mathcal{W}_{\mathrm{task}}$.}
Statistics preserving $P_{S,Y}$:
We include all 2-way marginals $(X_j,Y)$ for $j\in S$, optionally add a small set of 3-way marginals $(X_j,X_{j'},Y)$ for informative pairs selected via DP conditional mutual information, and (when feasible) include the full joint $(X_S,Y)$. Default: $m_3\le 20$, $\varepsilon_{\mathrm{mi}}=0.05\varepsilon$.

\paragraph{Background workload $\mathcal{W}_{\mathrm{bg}}$.}
Statistics for plausible overall structure:
We include all 1-way marginals and a spanning set of 2-way marginals $(X_j,X_{j'})$ selected by a DP MST-style procedure \citep{mckenna2021nist}.

\paragraph{Why this structure?}
Including 2-way and 3-way marginals over $S \cup \{Y\}$ ensures the PGM fitting stage has sufficient constraints to reconstruct $P_{S,Y}$ with low error (see Proposition~\ref{prop:chowliu} for a formal statement in the tree-structured case). The background workload maintains realistic marginal distributions for features outside $S$, supporting secondary analyses.

\subsection{Step 3: budget allocation}
Split the total budget $(\varepsilon, \delta)$ into four pools:
\begin{align*}
(\varepsilon,\delta) &= (\varepsilon_{\mathrm{sel}},0)+(\varepsilon_{\mathrm{mi}},0)\\
&\quad +(\varepsilon_{\mathrm{task}},\delta_{\mathrm{task}})+(\varepsilon_{\mathrm{bg}},\delta_{\mathrm{bg}}).
\end{align*}
Default split: $\varepsilon_{\mathrm{sel}} = 0.1\varepsilon$ (Regime 3 only), $\varepsilon_{\mathrm{mi}} = 0.05\varepsilon$ (for private MI estimation and 3-way marginal selection), $\varepsilon_{\mathrm{task}} = 0.65\varepsilon$, $\varepsilon_{\mathrm{bg}} = 0.2\varepsilon$, with $\delta$ split proportionally across measurement steps. In Regimes 1 and 2, the selection budget is reallocated to task measurement.

Within each pool, allocate to individual queries using Theorem~\ref{thm:alloc}. Define task weights $w_q$ for each $q \in \mathcal{W}_{\mathrm{task}}$:
We set $w_q=1$ for 2-way marginals $(X_j,Y)$, scale 3-way marginal weights by the DP-estimated conditional MI $\widetilde{I}(X_j;X_{j'}\mid Y)$ (normalized to average 1), and set $w_q=|S|$ for the full joint if included.
Background workload queries receive uniform weights $w_q = 1$.

Given weights $w_q$ and domain sizes $|\Omega_q|$, Theorem~\ref{thm:alloc} yields per-query budgets $\varepsilon_q^\star \propto \sqrt{w_q |\Omega_q|}$.

\subsection{Step 4: measurement and synthesis}
Measure each marginal $q \in \mathcal{W}$ by computing the empirical frequency vector (counts divided by $n$) and adding Gaussian noise calibrated to the allocated budget $\varepsilon_q$. Under replace-one neighbors, a probability vector changes by $+1/n$ and $-1/n$ in two cells, so $\ell_2$ sensitivity is $\Delta_2 = \sqrt{2}/n$. For $(\varepsilon_q, \delta_q)$-DP via the Gaussian mechanism, we add $\mathcal{N}(0, \sigma_q^2 I)$ where $\sigma_q = \Delta_2 \sqrt{2\log(1.25/\delta_q)}/\varepsilon_q$.

Fit a joint distribution $\widehat{P}$ via Private-PGM \citep{mckenna2019pgm}: this solves a maximum-entropy problem over distributions whose marginals match the noisy measurements, using mirror descent on a graphical model parametrization. The graph structure is determined by the workload---nodes are variables, and cliques correspond to measured marginals.

Finally, sample $n_{\mathrm{syn}}$ i.i.d.\ records from $\widehat{P}$ to produce $\widetilde{D}$. Both fitting and sampling are post-processing of DP outputs and incur no additional privacy cost.

\paragraph{Computational complexity.}
Feature selection (Step 1) requires $O(kd)$ score evaluations. Workload construction (Step 2) involves $O(|S|^2)$ marginals. PGM fitting (Step 4) scales as $O(T \cdot |\mathcal{W}| \cdot \bar{|\Omega|})$ for $T$ mirror-descent iterations and average domain size $\bar{|\Omega|}$; this is typically $O(10^4)$ for moderate workloads. Sampling is $O(n_{\mathrm{syn}} \cdot d)$. Overall runtime is dominated by PGM fitting and is comparable to existing marginal-based synthesizers.
Pseudocode for PRISM is provided in Appendix~\ref{sec:supp-theory}.

% ============================================================
\section{Theory: From DP Measurements to Predictive Utility}
\label{sec:theory}

This section develops theoretical guarantees connecting our mechanism design to downstream predictive performance. All proofs are deferred to Appendix~\ref{sec:proofs}.

\paragraph{Scope.}
We provide: (i) end-to-end DP via composition, (ii) distribution-free bounds connecting TV error on $(X_S,Y)$ to prediction risk (including ERM excess risk), and (iii) closed-form optimal allocation. Concrete instantiations appear in Appendix~\ref{sec:supp-theory}. Our bounds can be loose for structured distributions and apply to predictors using only $X_S$; we do not analyze PGM fitting rates \citep{mckenna2019pgm}.

\subsection{End-to-end DP guarantee}
\begin{theorem}[DP of PRISM]
\label{thm:dp}
Assume Step 1 (subset selection) is $(\varepsilon_{\mathrm{sel}},0)$-DP, Step 2 (MI estimation for 3-way marginal selection) is $(\varepsilon_{\mathrm{mi}},0)$-DP, and Step 4 measures each query $q\in\mathcal{W}$ with a mechanism that is $(\varepsilon_q,\delta_q)$-DP. Then the overall mechanism in Algorithm~\ref{alg:prism} is
\[
\Big(\varepsilon_{\mathrm{sel}}+\varepsilon_{\mathrm{mi}}+\sum_{q\in\mathcal{W}}\varepsilon_q,\ \sum_{q\in\mathcal{W}}\delta_q\Big)\text{-DP}.
\]
\end{theorem}

\begin{remark}
In practice, R\'enyi/zCDP accountants yield tighter bounds \citep{zhang2021privsyn,mckenna2022aim}; this changes only the allocation constraint, not the principle.
\end{remark}

\subsection{Graphical sufficiency for prediction via Markov blanket}
The following proposition formalizes the key graphical property motivating our workload design in Regime 2.

\begin{proposition}[Markov blanket sufficiency for prediction]
\label{prop:mb-sufficiency}
Let $(X,Y)$ be distributed according to a distribution $P$ that is Markov with respect to a DAG $G$. Then
\[
P(Y \mid X) = P(Y \mid X_{\MB(Y)}) \quad \text{(a.s.)}.
\]
Consequently, for any proper loss (e.g., log loss), the Bayes-optimal predictor can be expressed as a function of $X_{\MB(Y)}$ only.
\end{proposition}

\begin{remark}[Graphical vs.\ causal sufficiency]
\label{rem:graphical-vs-causal}
Proposition~\ref{prop:mb-sufficiency} is graphical, not causal: it guarantees prediction optimality in a fixed distribution but not shift robustness. For robustness, Regime~1 targets causal parents $\Pa(Y)$, whose mechanism tends to remain stable under spurious-correlation shifts (Section~\ref{sec:scm-results}). For DP allocation, the proposition implies that preserving $(X_{\MB(Y)},Y)$ suffices to preserve the Bayes rule---justifying preferential budget allocation.
\end{remark}

\subsection{Risk sensitivity to marginal closeness}
The next two lemmas establish that prediction risk depends only on the marginal $(X_S, Y)$ when the predictor uses only features in $S$.

\begin{lemma}[Risk depends only on $(X_S,Y)$ marginal]
\label{lem:depends-only-marg}
Fix a subset $S\subseteq[d]$ and a predictor $h(x)=h(x_S)$ that depends only on $x_S$. Then for any distribution $P$ over $(X,Y)$,
\[
\mathcal{R}_P(h) = \E_{(X_S,Y)\sim P_{S,Y}}[\ell(h(X_S),Y)].
\]
\end{lemma}

\begin{lemma}[Total variation controls risk difference]
\label{lem:tv-risk}
Let $\ell(\cdot,\cdot)\in[0,1]$ be bounded. For any two distributions $P,Q$ over $(X,Y)$, any subset $S$, and any predictor $h(x)=h(x_S)$,
\[
\big|\mathcal{R}_P(h) - \mathcal{R}_Q(h)\big|
\le \TV(P_{S,Y},Q_{S,Y}).
\]
\end{lemma}

\subsection{Learning from synthetic data: an excess risk bound}
Lemma~\ref{lem:tv-risk} is for a fixed $h$. We now connect it to learning on synthetic data.

Let $\Hcal_S := \{h: \Xcal \to \Ycal \mid h(x)=h(x_S)\}$ be the class restricted to features $S$.

\begin{theorem}[Excess risk of ERM on synthetic data]
\label{thm:erm}
Assume $\ell\in[0,1]$ and the downstream analyst trains by empirical risk minimization (ERM) over $\Hcal_S$ on a synthetic dataset $\widetilde{D}$ of size $n_{\mathrm{syn}}$ sampled i.i.d.\ from $\widetilde{P}$. Let $\widehat{h} \in \argmin_{h\in\Hcal_S}\widehat{\mathcal{R}}_{\widetilde{D}}(h)$. Then with probability at least $1-\beta$ over $\widetilde{D}$,
\[
\mathcal{R}_P(\widehat{h})
\le
\inf_{h\in\Hcal_S}\mathcal{R}_P(h)
\;+\; 2\Delta_S(P,\widetilde{P})
\;+\; 2\,\mathfrak{R}_{n_{\mathrm{syn}}}(\Hcal_S)
\;+\; 2\sqrt{\frac{\log(2/\beta)}{2n_{\mathrm{syn}}}},
\]
where $\mathfrak{R}_{n_{\mathrm{syn}}}(\Hcal_S)$ is the Rademacher complexity term controlling uniform convergence under $\widetilde{P}$.
\end{theorem}

\begin{remark}
The bound separates generalization under $\widetilde{P}$ from distribution shift $\Delta_S(P,\widetilde{P})$; our mechanism targets the latter via budget allocation.
\end{remark}

\paragraph{Concrete instantiations (appendix).}
Appendix~\ref{sec:supp-theory} instantiates these bounds for Gaussian histogram measurement, relating privacy budget to total variation error on $(X_S,Y)$, and provides a tree-structured reconstruction bound (Chow--Liu style) connecting marginal measurement error to joint distribution error.

\subsection{Optimal budget allocation for a risk upper bound}
\label{sec:alloc}
Having established how marginal error affects risk, we now turn to the \emph{budget allocation} problem: how should privacy budget be divided across measurements to minimize the risk bound?

Suppose we measure $m$ query groups (e.g., marginals) indexed by $t=1,\dots,m$, where group $t$ has domain size $|\Omega_t|$ (number of bins/cells). Let the DP noise scale behave as $\sigma_t = c_t/\varepsilon_t$ for constants $c_t$ determined by sensitivity and $\delta_t$. Suppose our downstream risk bound implies
\[
\Delta_S(P,\widetilde{P}) \;\le\; \sum_{t=1}^m w_t \cdot |\Omega_t| \cdot \sigma_t
\;=\; \sum_{t=1}^m \frac{a_t}{\varepsilon_t},
\quad \text{where } a_t := w_t|\Omega_t|c_t.
\]
Here $w_t$ is a \emph{task weight} capturing the importance of query group $t$ for predicting $Y$ (e.g., $w_t$ larger for marginals involving $\MB(Y)$ or top-ranked predictive features).

Consider the constrained optimization
\[
\min_{\varepsilon_1,\dots,\varepsilon_m>0}\ \sum_{t=1}^m \frac{a_t}{\varepsilon_t}
\quad \text{s.t.}\quad \sum_{t=1}^m \varepsilon_t \le \varepsilon_{\mathrm{meas}}.
\]

\begin{theorem}[Closed-form optimal allocation]
\label{thm:alloc}
The optimal solution satisfies
\[
\varepsilon_t^\star = \varepsilon_{\mathrm{meas}}
\cdot \frac{\sqrt{a_t}}{\sum_{s=1}^m \sqrt{a_s}},
\qquad t=1,\dots,m,
\]
and the optimal value is
\[
\sum_{t=1}^m \frac{a_t}{\varepsilon_t^\star}
=
\frac{1}{\varepsilon_{\mathrm{meas}}}\Big(\sum_{t=1}^m \sqrt{a_t}\Big)^2.
\]
\end{theorem}

\begin{remark}
The constant $a_t$ grows with domain size $|\Omega_t|$ and task weight $w_t$; the allocation $\propto\sqrt{a_t}$ directs budget toward task-critical marginals.
\end{remark}

% ============================================================
\section{Experiments}
\label{sec:experiments}

We evaluate PRISM on (i) a semi-synthetic structural causal model (SCM) benchmark where the ground-truth parents $\Pa(Y)$ and Markov blanket $\MB(Y)$ are known and (ii) the Adult Income dataset. Our experiments address three questions: (1) When does causal targeting (Regime~1) improve robustness under shift? (2) When does targeting the full Markov blanket (Regime~2) help when predictive children are stable? (3) When can closed-form optimal allocation (Theorem~\ref{thm:alloc}) provide gains over uniform allocation?

% Generated by: python experiments/paper1/run_full_benchmark.py --benchmark scm_spurious,scm_marginal,allocation_win
%             python experiments/paper1/run_icml_benchmark.py --datasets adult

\paragraph{Experimental setup.}
On the SCM benchmark, we generate discrete features with causal parents $A,B$, spurious children $S_1,\dots,S_{10}$, and additional irrelevant noise variables. We evaluate two distribution shifts: a \emph{spurious-correlation shift} that breaks the children at test time while preserving $P(Y\mid A,B)$, and a \emph{marginal/covariate shift} that changes $P(A,B)$ while preserving $P(Y\mid A,B)$. On Adult, we discretize continuous features into 8 quantile bins. Full dataset/preprocessing and privacy-accounting details are in the Extended Experiments appendix (Appendix~\ref{sec:extended-experiments}), with the evaluation protocol in Appendix~\ref{sec:ext-eval}.

\paragraph{Methods and metric.}
On SCM we report PRISM-Causal (Regime~1; $\Pa(Y)$) and PRISM-Graphical (Regime~2; $\MB(Y)$), each with closed-form (Opt) and uniform (Unif) allocation. On Adult we report PRISM-Predictive (Regime~3; DP feature selection) and its uniform-allocation ablation. We compare against MST (Private-PGM), PrivBayes (DataSynthesizer), and an oracle correlation top-$k$ baseline. Utility is measured via train-on-synthetic, test-on-real ROC-AUC (TSTR) using logistic regression. All SCM and Adult results are averaged over 10 random seeds; error bars show 95\% confidence intervals.

\subsection{SCM Spurious Shift Benchmark}
\label{sec:scm-results}

% Generated by: python experiments/paper1/run_full_benchmark.py --benchmark scm_spurious

Figure~\ref{fig:scm-shift} reports TSTR performance when spurious child features break at test time; Appendix~\ref{sec:extended-experiments} includes the corresponding $\varepsilon=1.0$ table.

\begin{figure}[t]
\centering
\includegraphics[width=\linewidth]{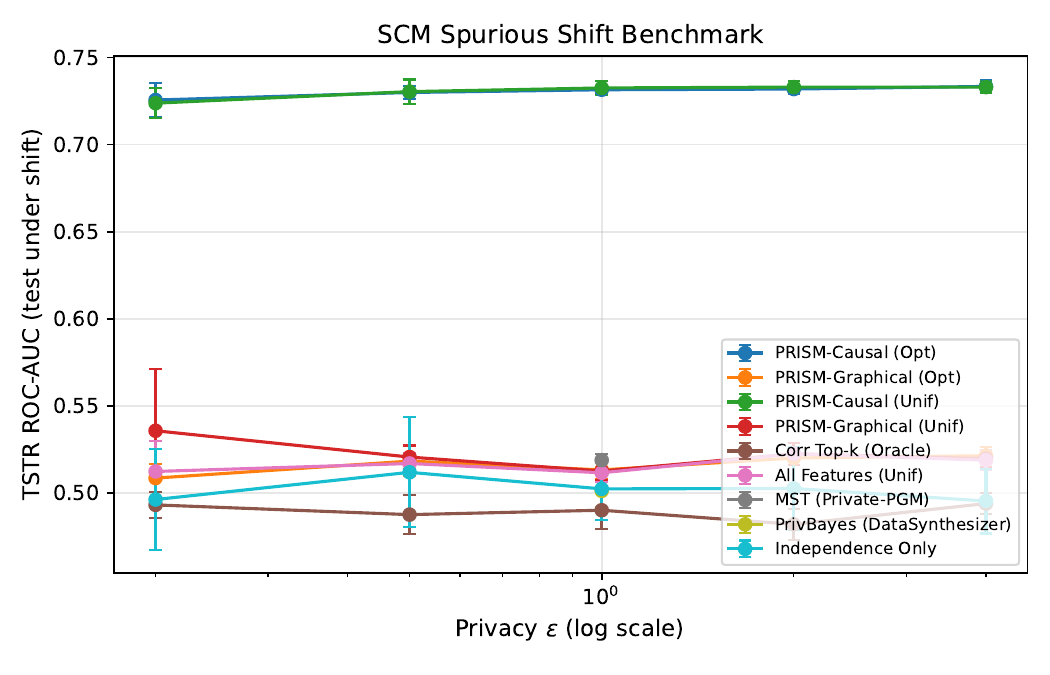}
\caption{SCM spurious shift: TSTR ROC-AUC vs.\ $\varepsilon$. PRISM-Causal (Regime~1) remains robust; correlation-based methods collapse because they rely on spurious children. Error bars: 95\% CI; MST/PrivBayes at $\varepsilon=1.0$ only.}
\label{fig:scm-shift}
\end{figure}

\paragraph{Causal guidance is critical under shift.}
Methods using the causal parents ($A, B$) achieve AUC $\approx 0.72$, closely matching the non-private upper bound. In contrast, correlation-based top-$k$ selection picks spurious correlates $S_1, \ldots, S_k$ (which have higher correlation with $Y$ in training data) and achieves AUC $\approx 0.49$---essentially random guessing. This demonstrates that \emph{causal structure, not correlation strength, determines robustness under distribution shift}.

\paragraph{Workload matters more than allocation in this setting.}
PRISM-Causal (Opt) and PRISM-Causal (Unif) are nearly indistinguishable: with only two parent features in the task workload, uniform and optimal allocation allocate similar noise levels. This indicates that in this setting, \emph{choosing the right features} is the dominant driver of robustness.

\paragraph{Blanket vs.\ parents.}
PRISM-Graphical (Regime~2; $\MB(Y)$) performs poorly here because it explicitly preserves the spurious children; when these relationships break at test time, the measured blanket becomes misleading. For prediction under spurious shifts, the minimal causal set $\Pa(Y)$ is preferable.

\subsection{SCM Marginal Shift Benchmark}
\label{sec:marginal-shift}

% Generated by: python experiments/paper1/run_full_benchmark.py --benchmark scm_marginal

We also evaluate a covariate/marginal shift where $P(A, B)$ changes at test time but $P(Y \mid A, B)$ is preserved (Figure~\ref{fig:scm-marginal}).

\begin{figure}[t]
\centering
\includegraphics[width=\linewidth]{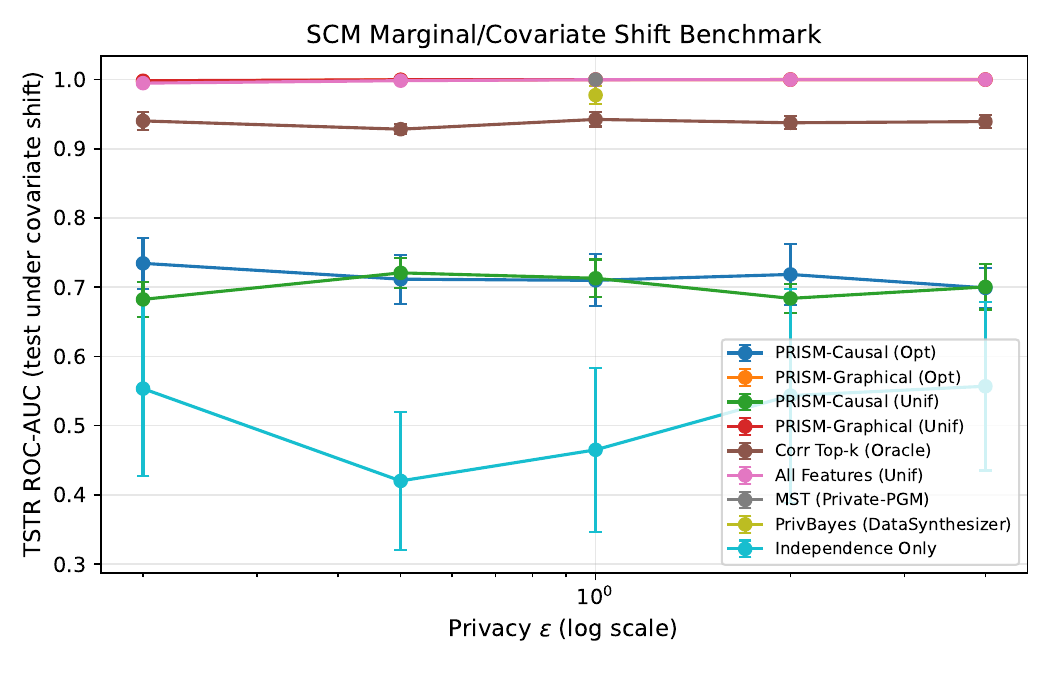}
\caption{SCM marginal shift: TSTR ROC-AUC vs.\ $\varepsilon$. Child features remain predictive, so blanket preservation helps. Error bars: 95\% CI; MST/PrivBayes at $\varepsilon=1.0$ only.}
\label{fig:scm-marginal}
\end{figure}

\paragraph{Blanket can dominate when children are stable.}
Because the child features $S_j$ remain strongly predictive of $Y$ in this shift, methods that preserve their relationship with $Y$---PRISM-Graphical and generic synthesizers such as MST/PrivBayes---achieve near-perfect AUC.

\paragraph{Parents provide robustness but may trade off accuracy.}
PRISM-Causal targets only $\Pa(Y)$ and thus ignores highly predictive children, yielding lower AUC in this setting. Together with the spurious-shift results, this illustrates why the appropriate regime depends on shift assumptions: causal parents improve robustness when spurious correlations may break, whereas blanket targeting can be advantageous when child relationships are stable.

\subsection{Real Data: Adult Income Benchmark}
\label{sec:adult-results}

% Generated by: python experiments/paper1/run_icml_benchmark.py --datasets adult

To complement the semi-synthetic SCM benchmark with a real dataset, we evaluate on Adult Income (11 features after preprocessing). Details of preprocessing and privacy accounting are in Appendix~\ref{sec:extended-experiments}.

\begin{figure}[t]
\centering
\includegraphics[width=\linewidth]{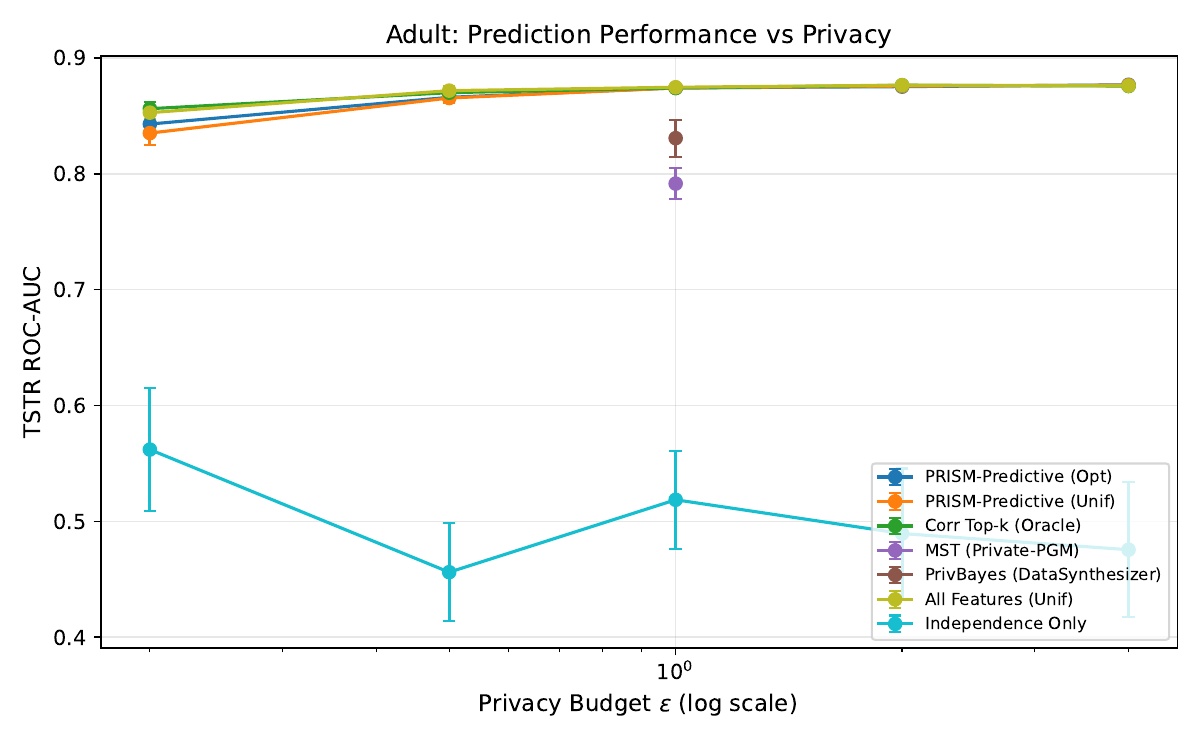}
\caption{Adult income: TSTR ROC-AUC vs.\ $\varepsilon$. Workload-based variants are similar on this large dataset; MST/PrivBayes lag. Error bars: 95\% CI; MST/PrivBayes at $\varepsilon=1.0$ only.}
\label{fig:adult}
\end{figure}

\paragraph{Task-aware synthesis matches strong internal baselines.}
PRISM-Predictive (Opt/Unif) is close to \emph{All Features (Unif)} on this low-dimensional dataset, indicating that feature selection and allocation have limited headroom when $d$ is small and $n$ is large. However, both PRISM variants are substantially better than MST and PrivBayes at the same privacy level.

\paragraph{When does optimal allocation help?}
In our SCM and Adult experiments, the gap between Opt and Unif allocation is modest: the SCM parents workload is tiny, and Adult is large enough that DP noise is relatively small. Appendix~\ref{sec:alloc-win} constructs a controlled setting with heterogeneous task weights where closed-form allocation yields a large gain, consistent with Theorem~\ref{thm:alloc}.

% ============================================================
\section{Limitations and Discussion}
\label{sec:limitations}

\paragraph{Causal graph availability.}
Our strongest causal results assume $\MB(Y)$ is known. In practice, domain knowledge may specify partial causal structure; or one may use DP causal discovery \citep{wang2020privpc}, which introduces additional privacy and estimation error. We treat ``DP-learned blanket'' as a secondary variant and leave its empirical evaluation to future work.

\paragraph{DP feature selection overhead.}
In Regime~3, PRISM privately estimates feature--target mutual information to select a top-$k$ subset and to derive allocation weights. This consumes privacy budget and introduces estimation noise, which can lead to suboptimal selections or misweighted allocations in small-data regimes. When public pilot data or reliable domain knowledge is available, one can reduce this cost by fixing a feature set or using public importance estimates.

\paragraph{High-cardinality and continuous features.}
Marginal-based methods can struggle with high-cardinality columns due to domain explosion; discretization introduces bias. Mixed-type approaches like RAP++ \citep{vietri2022rappp} suggest a promising direction. Our framework is compatible with mixed-type workloads by replacing histogram marginals with appropriate threshold/moment queries; the core allocation logic remains.

\paragraph{Beyond a single target $Y$.}
This paper focuses on a designated target $Y$. Extending to multiple targets (multi-task) is natural, but requires careful budget allocation and is closer to the setting studied by \citet{vietri2022rappp}. We treat multi-target as future work or as an ablation.

\paragraph{Not a replacement for task-agnostic release.}
If the use case truly requires broad exploratory analysis across all columns, a task-agnostic synthesizer may be preferable. Our method is designed for cases where the curator can legitimately prioritize a known primary prediction task.

\paragraph{Iterative and adaptive queries.}
Our task-derived workload assumes a fixed prediction target specified upfront. In practice, analysts may need to query data or refine models iteratively, where the sequence of queries can influence what analyses remain supportable under a fixed privacy budget. This challenge---adaptive data analysis under DP---is addressed by techniques such as the reusable holdout \citep{dwork2015reusable}, which provides a principled way to answer adaptively chosen queries. Integrating such adaptive mechanisms with task-aware synthesis is a promising direction for future work.

% ============================================================
\section{Conclusion}
\label{sec:conclusion}

We proposed PRISM, a prediction-centric framework for differentially private synthetic data release that operates across three regimes depending on available structural knowledge. The key idea is to treat privacy budget as a scarce resource and allocate it preferentially to preserve the dependence structure most relevant for predicting a designated target variable $Y$. In the \emph{causal regime}, when parents of $Y$ are known and distribution shift is expected, we target parents for robustness. In the \emph{graphical regime}, when a Bayesian network structure is available and the distribution is stable, the Markov blanket provides a sufficient feature set for optimal prediction. In the \emph{predictive regime}, when no structural knowledge exists, we select features via DP methods without claiming to recover any particular structure.

Experiments support the regime view: on the SCM benchmark with spurious shifts, PRISM-Causal (Regime~1) maintains high AUC while correlation-based selection and generic synthesizers fail; under marginal shift, PRISM-Graphical (Regime~2) benefits from stable predictive children and can achieve near-perfect AUC, illustrating a practical parents-vs.-blanket tradeoff. On Adult, PRISM-Predictive (Regime~3) is competitive with strong internal baselines and improves over MST/PrivBayes at the same privacy level. Finally, Appendix~\ref{sec:alloc-win} demonstrates that closed-form allocation can yield substantial gains when task weights are heterogeneous, especially under tight privacy.

For completeness, the appendix contains: Supplementary Theory (Appendix~\ref{sec:supp-theory}) with pseudocode and additional instantiations; Proofs (Appendix~\ref{sec:proofs}); and Extended Experiments (Appendix~\ref{sec:extended-experiments}) with detailed setup, full tables, and additional diagnostics.

% ============================================================
\section*{Broader Impact}

This work addresses privacy-preserving data sharing, a problem with direct societal relevance. By enabling organizations to release synthetic datasets that preserve predictive utility while satisfying differential privacy, our method can facilitate research collaboration in sensitive domains such as healthcare and finance where raw data sharing is infeasible.

Potential risks include over-reliance on synthetic data without understanding its limitations: models trained on synthetic data may underperform on subpopulations poorly represented in the original dataset, and the task-specific design means the released data may not support analyses beyond the designated prediction task. We encourage practitioners to clearly document the intended use case and privacy parameters when releasing synthetic data, and to validate downstream models on held-out real data when possible.

\bibliographystyle{plainnat}
\bibliography{ref1}

@inproceedings{dwork2006calibrating,
  title={Calibrating Noise to Sensitivity in Private Data Analysis},
  author={Dwork, Cynthia and McSherry, Frank and Nissim, Kobbi and Smith, Adam},
  booktitle={Theory of Cryptography Conference (TCC)},
  year={2006},
  pages={265--284},
  publisher={Springer}
}

@article{dwork2014algorithmic,
  title={The Algorithmic Foundations of Differential Privacy},
  author={Dwork, Cynthia and Roth, Aaron},
  journal={Foundations and Trends in Theoretical Computer Science},
  volume={9},
  number={3--4},
  pages={211--407},
  year={2014},
  publisher={Now Publishers}
}

@inproceedings{mcsherry2007mechanism,
  title={Mechanism Design via Differential Privacy},
  author={McSherry, Frank and Talwar, Kunal},
  booktitle={Foundations of Computer Science (FOCS)},
  year={2007},
  pages={94--103},
  publisher={IEEE}
}

@inproceedings{hardt2012mwem,
  title={A Simple and Practical Algorithm for Differentially Private Data Release},
  author={Hardt, Moritz and Ligett, Katrina and McSherry, Frank},
  booktitle={Advances in Neural Information Processing Systems (NeurIPS)},
  year={2012},
  pages={2339--2347}
}

@inproceedings{gaboardi2014dualquery,
  title={Practical Private Query Release for High Dimensional Data},
  author={Gaboardi, Marco and Lim, Hyun-Woo and Rogers, Ryan and Vadhan, Salil and others},
  booktitle={Advances in Neural Information Processing Systems (NeurIPS)},
  year={2014},
  note={arXiv:1402.1526}
}

@inproceedings{mckenna2019pgm,
  title={Graphical-model Based Estimation and Inference for Differential Privacy},
  author={McKenna, Ryan and Sheldon, Daniel and Miklau, Gerome},
  booktitle={International Conference on Machine Learning (ICML)},
  year={2019},
  pages={4435--4444},
  publisher={PMLR}
}

@article{zhang2017privbayes,
  title={PrivBayes: Private Data Release via Bayesian Networks},
  author={Zhang, Jun and Cormode, Graham and Procopiuc, Cecilia M and Srivastava, Divesh and Xiao, Xiaokui},
  journal={ACM Transactions on Database Systems (TODS)},
  volume={42},
  number={4},
  pages={25:1--25:41},
  year={2017},
  publisher={ACM},
  doi={10.1145/3134428}
}

@article{mckenna2021nist,
  title={Winning the {NIST} Contest: A Scalable and General Approach to Differentially Private Synthetic Data},
  author={McKenna, Ryan and Miklau, Gerome and Sheldon, Daniel},
  journal={arXiv preprint},
  year={2021},
  eprint={2108.04978},
  archivePrefix={arXiv},
  primaryClass={cs.LG}
}

@article{mckenna2022aim,
  title={{AIM}: An Adaptive and Iterative Mechanism for Differentially Private Synthetic Data},
  author={McKenna, Ryan and Mullins, Brett and Sheldon, Daniel and Miklau, Gerome},
  journal={Proceedings of the VLDB Endowment (PVLDB)},
  volume={15},
  number={11},
  pages={2599--2612},
  year={2022},
  doi={10.14778/3551793.3551817}
}

@inproceedings{zhang2021privsyn,
  title={PrivSyn: Differentially Private Data Synthesis},
  author={Zhang, Zhikun and Wang, Tianhao and Li, Ninghui and Honorio, Jean and Backes, Michael and He, Shibo and Chen, Jiming and Zhang, Yang},
  booktitle={USENIX Security Symposium},
  year={2021}
}

@article{jordon2019pategan,
  title={PATE-GAN: Generating Synthetic Data with Differential Privacy Guarantees},
  author={Jordon, James and Yoon, Jinsung and van der Schaar, Mihaela},
  journal={International Conference on Learning Representations (ICLR)},
  year={2019},
  note={arXiv:1806.01860}
}

@inproceedings{fang2022dpctgan,
  title={DP-CTGAN: Differentially Private Medical Data Generation Using {CTGAN}s},
  author={Fang, Mei Ling and Dhami, Devendra Singh and Kersting, Kristian},
  booktitle={Artificial Intelligence in Medicine},
  series={Lecture Notes in Computer Science},
  pages={178--188},
  publisher={Springer},
  year={2022},
  doi={10.1007/978-3-031-09342-5_17}
}

@article{su2015privpfc,
  title={Differentially Private Projected Histograms of Multi-Attribute Data for Classification},
  author={Su, Dong and Cao, Jianneng and Li, Ninghui},
  journal={arXiv preprint},
  year={2015},
  eprint={1504.05997},
  archivePrefix={arXiv},
  primaryClass={cs.DB}
}

@inproceedings{vietri2022rappp,
  title={Private Synthetic Data for Multitask Learning and Marginal Queries},
  author={Vietri, Giuseppe and Archambeau, Cedric and Aydore, Sergul and Brown, William and Kearns, Michael and Roth, Aaron and Siva, Ankit and Tang, Shuai and Wu, Zhiwei Steven},
  booktitle={Advances in Neural Information Processing Systems (NeurIPS)},
  year={2022}
}

@inproceedings{thakurta2013stability,
  title={Differentially Private Feature Selection via Stability Arguments, and the Robustness of the Lasso},
  author={Thakurta, Abhradeep and Smith, Adam},
  booktitle={Conference on Learning Theory (COLT)},
  year={2013},
  publisher={PMLR}
}

@inproceedings{wang2020privpc,
  title={Towards Practical Differentially Private Causal Graph Discovery},
  author={Wang, Lun and Pang, Qi and Song, Dawn},
  booktitle={Advances in Neural Information Processing Systems (NeurIPS)},
  year={2020},
  note={Also circulated as arXiv:2006.08598}
}

@book{koller2009pgm,
  title={Probabilistic Graphical Models: Principles and Techniques},
  author={Koller, Daphne and Friedman, Nir},
  publisher={MIT Press},
  year={2009}
}

@book{pearl2009causality,
  title={Causality: Models, Reasoning, and Inference},
  author={Pearl, Judea},
  edition={2},
  publisher={Cambridge University Press},
  year={2009}
}

@inproceedings{abadi2016dpsgd,
  title={Deep Learning with Differential Privacy},
  author={Abadi, Martin and Chu, Andy and Goodfellow, Ian and McMahan, H. Brendan and Mironov, Ilya and Talwar, Kunal and Zhang, Li},
  booktitle={ACM SIGSAC Conference on Computer and Communications Security (CCS)},
  year={2016},
  pages={308--318},
  publisher={ACM},
  doi={10.1145/2976749.2978318}
}

@article{dwork2015reusable,
  title={The Reusable Holdout: Preserving Validity in Adaptive Data Analysis},
  author={Dwork, Cynthia and Feldman, Vitaly and Hardt, Moritz and Pitassi, Toniann and Reingold, Omer and Roth, Aaron},
  journal={Science},
  volume={349},
  number={6248},
  pages={636--638},
  year={2015},
  publisher={American Association for the Advancement of Science},
  doi={10.1126/science.aaa9375}
}

% ============================================================
\appendix
\section{Supplementary Theory}
\label{sec:supp-theory}

\subsection{Algorithmic description}
\begin{algorithm}[t]
\caption{PRISM: Structure-Aware DP Synthesis}
\label{alg:prism}
\begin{algorithmic}[1]
\STATE \textbf{Input:} dataset $D$, target $Y$, privacy $(\varepsilon,\delta)$, regime $\in \{1,2,3\}$, optional structure ($\Pa(Y)$ or $G$), subset size $k$ (Regime 3), max marginal order $r$, $n_{\mathrm{syn}}$
\STATE Split privacy: $(\varepsilon,\delta) \leftarrow (\varepsilon_{\mathrm{sel}},0)+(\varepsilon_{\mathrm{mi}},0)+(\varepsilon_{\mathrm{task}},\delta_{\mathrm{task}})+(\varepsilon_{\mathrm{bg}},\delta_{\mathrm{bg}})$
\IF{Regime 1 (causal)}
\STATE $S \leftarrow \Pa(Y)$ \COMMENT{Parents for shift robustness}
\ELSIF{Regime 2 (graphical)}
\STATE $S \leftarrow \MB(Y)$ \COMMENT{Blanket for prediction optimality}
\ELSE
\STATE Privately select $S$ of size $k$ via exponential mechanism or stability-based DP feature selection
\ENDIF
\STATE Construct workload $\mathcal{W} \leftarrow \mathcal{W}_{\mathrm{task}}(S,Y,r)\;\cup\;\mathcal{W}_{\mathrm{bg}}(r)$
\STATE Allocate per-query budgets $\{(\varepsilon_q,\delta_q)\}_{q\in\mathcal{W}}$ using Section~\ref{sec:alloc}
\STATE Measure each $q(D)$ with Gaussian mechanism to obtain noisy answers $\{\widetilde{q}\}_{q\in\mathcal{W}}$
\STATE Fit distribution $\widehat{P}$ via PGM inference / maximum entropy consistent with $\{\widetilde{q}\}$ \citep{mckenna2019pgm}
\STATE Sample $\widetilde{D} \sim \widehat{P}^{\otimes n_{\mathrm{syn}}}$
\STATE \textbf{Return:} $\widetilde{D}$
\end{algorithmic}
\end{algorithm}

\subsection{Explicit bounds for histogram/marginal measurement}
To make the risk bounds concrete, consider discrete or binned data where the mechanism explicitly measures the histogram of $(X_S,Y)$.

Assume $(X_S,Y)$ takes values in a finite domain $\Omega$ of size $|\Omega| = |\Xcal_S|\cdot|\Ycal|$. Let the empirical histogram vector be
\[
\widehat{p} \in \Delta^{|\Omega|-1},\quad
\widehat{p}(\omega) := \frac{1}{n}\sum_{i=1}^n \mathbf{1}\{(x_{i,S},y_i)=\omega\}.
\]
A natural DP release is a noisy histogram $\widetilde{p} = \Pi_\Delta(\widehat{p} + \eta)$ where $\eta\sim \mathcal{N}(0,\sigma^2 I)$ and $\Pi_\Delta$ projects to the probability simplex (post-processing).

\begin{proposition}[High-probability $\ell_1$ error for Gaussian histogram]
\label{prop:l1hist}
Fix $\beta\in(0,1)$. With probability at least $1-\beta$,
\[
\|\widetilde{p} - \widehat{p}\|_1
\le
|\Omega| \cdot \sigma \sqrt{2\log(|\Omega|/\beta)}.
\]
Consequently,
\[
\TV(\widetilde{p},\widehat{p})
\le \frac12 |\Omega| \cdot \sigma \sqrt{2\log(|\Omega|/\beta)}.
\]
\end{proposition}

\begin{remark}[From empirical to population]
Proposition~\ref{prop:l1hist} bounds deviation from the empirical histogram. To relate to $P_{S,Y}$, add the usual sampling error $\|\widehat{p}-P_{S,Y}\|_1$ (concentration in multinomials). In many DP synthesis evaluations, the synthetic data is generated from the post-processed estimate and then compared to held-out test data, so both sources matter.
\end{remark}

\subsection{From privacy budget to histogram noise}
For the Gaussian mechanism with $\ell_2$ sensitivity $\Delta_2$, a standard calibration is
$\sigma \ge \Delta_2 \sqrt{2\log(1.25/\delta)}/\varepsilon$ (one common form) \citep{dwork2014algorithmic}.
For replace-one neighbors, the probability vector has $\Delta_2=\sqrt{2}/n$ (one record changes at most two cells by $1/n$ each, giving $\ell_2$ sensitivity $\sqrt{(1/n)^2 + (1/n)^2} = \sqrt{2}/n$). Hence:
\[
\sigma \asymp \frac{1}{n}\cdot \frac{\sqrt{\log(1/\delta)}}{\varepsilon}.
\]
Plugging into Proposition~\ref{prop:l1hist} yields an explicit scaling:
\[
\TV(P_{S,Y},\widetilde{P}_{S,Y}) \;\lesssim\; \frac{|\Omega|}{n}\cdot \frac{\sqrt{\log(1/\delta)\log(|\Omega|/\beta)}}{\varepsilon}
\quad \text{(up to constants + sampling error)}.
\]
This makes the central tension clear: large domain size $|\Omega|$ (high-dimensional $S$ or high-cardinality bins) increases error. Causal guidance helps precisely by shrinking $S$.

\subsection{Special case instantiations}
The general framework specializes to several common modeling regimes.

\subsubsection{Special case 1: Bayes-optimal prediction under Markov blanket}
If $S=\MB(Y)$ is known and the synthetic distribution satisfies $\Delta_{\MB(Y)}(P,\widetilde{P})\le \alpha$, then by Theorem~\ref{thm:erm} and Proposition~\ref{prop:mb-sufficiency}, learning over $\Hcal_{\MB(Y)}$ yields excess risk within $2\alpha$ (plus generalization terms) of the best possible predictor using all features.

\subsubsection{Special case 2: logistic regression with bounded features}
For binary $Y\in\{0,1\}$ and logistic regression $h_\theta(x_S)=\sigma(\theta^\top \phi(x_S))$ with bounded feature map $\|\phi(x_S)\|_2\le B$, one can derive bounds in terms of deviations in (i) $\E[\phi(X_S)Y]$ and (ii) $\E[\phi(X_S)\phi(X_S)^\top]$. Our method recovers this by using a workload consisting of the needed moments and allocating privacy budget to them. This is the continuous/moment analogue of the histogram instantiation above.

\subsubsection{Special case 3: squared-loss regression}
For $Y\in[-1,1]$ and linear predictors $h_\theta(x_S)=\theta^\top \phi(x_S)$ with $\|\theta\|_2\le R$, Lipschitz arguments show risk differences bounded by moment errors. Again, task-aware workload design consists of measuring only the moments involving selected features $S$.

\begin{remark}[TV-based theory vs.\ moment instantiations]
The TV-based theory is a clean, distribution-free statement that directly matches marginal-based synthesizers and discrete tabular settings. Moment-based instantiations matter in mixed-type settings and connect to methods such as RAP++ \citep{vietri2022rappp}; the allocation logic carries over by treating each moment/query family as a group in the allocation optimization.
\end{remark}

\subsection{Bridging marginal measurements to joint distribution error}
The preceding bounds relate privacy budget to marginal TV error. A remaining question is: how does error in \emph{measured} marginals propagate to error in the \emph{reconstructed} joint distribution $\widehat{P}$? For general graphical models, this depends on the structure and is analyzed in \citet{mckenna2019pgm}. Here we provide an explicit bound for tree-structured distributions, which clarifies the key dependencies.

\begin{proposition}[Chow--Liu reconstruction error (proof sketch)]
\label{prop:chowliu}
Let $(X_S, Y)$ have a joint distribution $P$ that factors according to a tree $T$ over variables $S \cup \{Y\}$. Suppose all univariate marginals satisfy $P_i(x) \ge p_{\min} > 0$ for all values $x$ in their support (bounded away from zero). Suppose we measure all pairwise marginals $P_{ij}$ for edges $(i,j) \in T$ and all univariate marginals $P_i$, each with TV error at most $\alpha < p_{\min}/2$. Let $\widehat{P}$ be the maximum-entropy distribution matching the noisy marginals (which, for trees, is the unique distribution factorizing as $\widehat{P}(x_S,y) = \prod_{(i,j)\in T} \widehat{P}_{ij}(x_i,x_j) / \prod_{i \in S \cup \{Y\}} \widehat{P}_i(x_i)^{d_i - 1}$, where $d_i$ is the degree of node $i$).

Then
\[
\TV(P, \widehat{P}) \;\le\; |T| \cdot \frac{\alpha}{p_{\min}} + O(\alpha^2/p_{\min}^2),
\]
where $|T| = |S|$ is the number of edges in the tree.
\end{proposition}

\begin{remark}[Implications for PRISM]
Proposition~\ref{prop:chowliu} shows that if the workload includes all pairwise marginals along a tree spanning $S \cup \{Y\}$, the joint reconstruction error scales linearly in $|S|$ times the per-marginal error. Combined with Proposition~\ref{prop:l1hist} and Theorem~\ref{thm:erm}, this yields an end-to-end bound from privacy budget to excess risk:
\[
\text{Excess risk} \;\lesssim\; |S| \cdot \frac{|\bar{\Omega}|}{n} \cdot \frac{\sqrt{\log(1/\delta)\log(|\bar{\Omega}|/\beta)}}{\varepsilon_{\text{task}}} + \text{(generalization terms)},
\]
where $|\bar{\Omega}|$ is the average pairwise domain size. For non-tree structures (cliques, higher-order marginals), the dependence on structure is more complex; see \citet{mckenna2019pgm} for junction-tree-based analysis.
\end{remark}

\section{Proofs}
\label{sec:proofs}

\subsection{Proof of Theorem~\ref{thm:dp} (DP of PRISM)}
\begin{proof}
Subset selection, MI estimation, and query measurements are applied sequentially to the same dataset. By the basic sequential composition theorem for DP \citep{dwork2014algorithmic}, privacy losses add: the combined mechanism is $(\sum \varepsilon_t, \sum \delta_t)$-DP. The PGM fitting and sampling are post-processing of the DP outputs and thus do not affect DP.
\end{proof}

\subsection{Proof of Proposition~\ref{prop:mb-sufficiency} (Markov blanket sufficiency)}
\begin{proof}
In a Bayesian network, the Markov blanket $\MB(Y)$ d-separates $Y$ from the remaining variables. By the global Markov property \citep{koller2009pgm}, $Y \perp\!\!\!\perp X_{\setminus \MB(Y)} \mid X_{\MB(Y)}$, which is equivalent to $P(Y\mid X)=P(Y\mid X_{\MB(Y)})$ almost surely. For proper losses, the Bayes predictor is a function of the conditional distribution of $Y$ given inputs, hence depends only on $X_{\MB(Y)}$.
\end{proof}

\subsection{Proof of Lemma~\ref{lem:depends-only-marg} (Risk depends on marginal)}
\begin{proof}
Since $h$ and $\ell(h(X),Y)$ depend only on $(X_S,Y)$, we can integrate out $X_{\setminus S}$; the expectation reduces to the marginal.
\end{proof}

\subsection{Proof of Lemma~\ref{lem:tv-risk} (TV controls risk difference)}
\begin{proof}
By Lemma~\ref{lem:depends-only-marg}, the risks are expectations of the same bounded function $f(x_S,y)=\ell(h(x_S),y)\in[0,1]$ under $P_{S,Y}$ and $Q_{S,Y}$. For bounded $f\in[0,1]$, the difference of expectations is upper bounded by total variation distance $\TV(P_{S,Y},Q_{S,Y})$ (equivalently by the variational definition of TV).
\end{proof}

\subsection{Proof of Theorem~\ref{thm:erm} (Excess risk of ERM on synthetic data)}
\begin{proof}
Standard uniform convergence (e.g., via Rademacher complexity) implies that with probability $\ge 1-\beta$ over the draw of $\widetilde{D}\sim \widetilde{P}^{\otimes n_{\mathrm{syn}}}$,
\[
\sup_{h\in\Hcal_S}\big|\mathcal{R}_{\widetilde{P}}(h)-\widehat{\mathcal{R}}_{\widetilde{D}}(h)\big|
\le \mathfrak{R}_{n_{\mathrm{syn}}}(\Hcal_S)+\sqrt{\frac{\log(2/\beta)}{2n_{\mathrm{syn}}}}.
\]
Denote $\gamma := \mathfrak{R}_{n_{\mathrm{syn}}}(\Hcal_S)+\sqrt{\log(2/\beta)/(2n_{\mathrm{syn}})}$. Since $\widehat{h}$ minimizes empirical risk on $\widetilde{D}$, we have for any $h^\star \in \Hcal_S$:
\begin{align*}
\mathcal{R}_{\widetilde{P}}(\widehat{h}) &\le \widehat{\mathcal{R}}_{\widetilde{D}}(\widehat{h}) + \gamma
\le \widehat{\mathcal{R}}_{\widetilde{D}}(h^\star) + \gamma
\le \mathcal{R}_{\widetilde{P}}(h^\star) + 2\gamma.
\end{align*}
Taking infimum over $h^\star$: $\mathcal{R}_{\widetilde{P}}(\widehat{h}) \le \inf_{h\in\Hcal_S}\mathcal{R}_{\widetilde{P}}(h) + 2\gamma$.

Now apply Lemma~\ref{lem:tv-risk} twice. For the left side: $\mathcal{R}_{P}(\widehat{h}) \le \mathcal{R}_{\widetilde{P}}(\widehat{h}) + \Delta_S(P,\widetilde{P})$. For the infimum: $\inf_{h}\mathcal{R}_{\widetilde{P}}(h) \le \inf_h \mathcal{R}_P(h) + \Delta_S(P,\widetilde{P})$ (since TV controls the difference for each $h$, it controls the infimum). Combining yields the stated bound with the factor of $2\Delta_S(P,\widetilde{P})$.
\end{proof}

\subsection{Proof of Proposition~\ref{prop:l1hist} ($\ell_1$ error for Gaussian histogram)}
\begin{proof}
Before projection, each coordinate noise $\eta_\omega$ satisfies
$|\eta_\omega| \le \sigma\sqrt{2\log(|\Omega|/\beta)}$
simultaneously for all $\omega$ with probability $\ge 1-\beta$ by a union bound on Gaussian tails.
Thus $\|\eta\|_1 \le |\Omega|\sigma\sqrt{2\log(|\Omega|/\beta)}$.

For the projection step: let $\widetilde{p}_{\mathrm{raw}} = \widehat{p} + \eta$ be the noisy histogram before projection. The simplex projection $\Pi_\Delta$ is the $\ell_2$-projection onto $\Delta^{|\Omega|-1}$. By the variational characterization of projections onto convex sets, for any $p \in \Delta^{|\Omega|-1}$:
\[
\|\Pi_\Delta(\widetilde{p}_{\mathrm{raw}}) - p\|_2 \le \|\widetilde{p}_{\mathrm{raw}} - p\|_2.
\]
Taking $p = \widehat{p}$ (which lies on the simplex): $\|\widetilde{p} - \widehat{p}\|_2 \le \|\eta\|_2$.

To convert to $\ell_1$: use $\|v\|_1 \le \sqrt{|\Omega|}\|v\|_2$ to get
\[
\|\widetilde{p} - \widehat{p}\|_1 \le \sqrt{|\Omega|}\|\eta\|_2 \le \sqrt{|\Omega|} \cdot \sigma\sqrt{|\Omega| \cdot 2\log(|\Omega|/\beta)} = |\Omega|\sigma\sqrt{2\log(|\Omega|/\beta)}.
\]
The TV bound follows from $\TV(p,q)=\frac12\|p-q\|_1$.
\end{proof}

\subsection{Proof of Theorem~\ref{thm:alloc} (Closed-form optimal allocation)}
\begin{proof}
Use a Lagrangian $\mathcal{L}(\varepsilon,\lambda)=\sum_t a_t/\varepsilon_t + \lambda(\sum_t \varepsilon_t-\varepsilon_{\mathrm{meas}})$.
Stationarity gives $-a_t/\varepsilon_t^2 + \lambda=0$, hence $\varepsilon_t=\sqrt{a_t/\lambda}$.
Enforce $\sum_t \varepsilon_t=\varepsilon_{\mathrm{meas}}$ to get $\sqrt{1/\lambda}=\varepsilon_{\mathrm{meas}}/(\sum_s \sqrt{a_s})$.
Substitute back for $\varepsilon_t^\star$ and the objective value.
\end{proof}

\subsection{Proof of Proposition~\ref{prop:chowliu} (Chow--Liu reconstruction error)}
\begin{proof}
The Chow--Liu factorization expresses the joint as a product over edge conditionals. Each edge contributes a factor $P(X_j \mid X_i) = P_{ij}(X_i, X_j)/P_i(X_i)$. TV error in $P_{ij}$ and $P_i$ propagates to error in this conditional; for $P_i \ge p_{\min}$ and perturbation $\alpha < p_{\min}/2$, the ratio error is $O(\alpha/p_{\min})$ per edge. Summing over $|T|$ edges and using sub-additivity of TV for product distributions yields the stated bound. The $O(\alpha^2/p_{\min}^2)$ term captures cross-terms. A full proof with explicit constants follows the analysis in \citet{koller2009pgm}, Chapter 11.
\end{proof}

\section{Extended Experiments}
\label{sec:extended-experiments}

This appendix provides the full experimental details for the results in Section~\ref{sec:experiments}, along with additional tables and diagnostics.

\subsection{Datasets and preprocessing}

\paragraph{Semi-synthetic SCM benchmark.}
We generate a discrete SCM with causal parents $A,B\in\{0,1,2\}$ sampled uniformly and target
$Y \sim \mathrm{Bernoulli}(\sigma(0.9(A-1)+0.9(B-1)+\eta))$ with $\eta\sim \mathcal{N}(0,0.5)$.
We create $10$ spurious child features $S_j \in \{0,1\}$ as $S_j = Y \oplus \mathrm{Bernoulli}(p_{\mathrm{flip}})$ and add $10$ irrelevant noise features $N_j \in \{0,1,2,3\}$ sampled independently.
Thus $d=22$ total features, $|\Pa(Y)|=2$, and $|\MB(Y)|=12$ (parents plus children).
We use $n_{\mathrm{train}}=5000$ and $n_{\mathrm{test}}=5000$.

\paragraph{Shifts.}
For the \emph{spurious shift} benchmark, we set $p_{\mathrm{flip}}^{\mathrm{train}}=0.10$ and $p_{\mathrm{flip}}^{\mathrm{test}}=0.50$, preserving $P(Y\mid A,B)$ but breaking child correlations.
For the \emph{marginal shift} benchmark, we shift the marginal distribution $P(A,B)$ at test time (towards higher values) while keeping $P(Y\mid A,B)$ fixed; we keep the child flip rate fixed at $0.15$ so that children remain predictive.

\paragraph{Adult Income.}
We use the numeric-encoded Adult dataset distributed with Private-PGM (48,842 rows).
We drop \texttt{fnlwgt} and \texttt{native-country}, discretize continuous features into 8 quantile bins, and perform an 80/20 stratified split, yielding $n_{\mathrm{train}}=39{,}073$ and $n_{\mathrm{test}}=9{,}769$ with $d=11$ features.

\subsection{Methods and privacy accounting}

\paragraph{PRISM variants (shared synthesis backend).}
All PRISM variants use the same Naive Bayes synthesis backend: we privately measure $P(Y)$ and a set of pairwise marginals $(X_j,Y)$ for a selected feature set $S$, then sample from $P(Y)\prod_{j\in S}P(X_j\mid Y)$ and sample remaining features independently from their (noisy) 1-way marginals.
PRISM-Causal (Regime~1) uses $S=\Pa(Y)$; PRISM-Graphical (Regime~2) uses $S=\MB(Y)$; PRISM-Predictive (Regime~3) selects $S$ via DP mutual information (top-$k$, with $k=8$ on Adult).

\paragraph{Closed-form vs.\ uniform allocation.}
PRISM-(Opt) allocates the task measurement budget across $(X_j,Y)$ marginals according to Theorem~\ref{thm:alloc}; PRISM-(Unif) allocates uniformly across the same workload. Unless otherwise noted, we use equal task weights on SCM and DP-estimated feature--target MI as weights on Adult.

\paragraph{Baselines.}
We compare against MST (Private-PGM) and PrivBayes (DataSynthesizer) as external end-to-end DP synthesizers, and include a non-private correlation top-$k$ feature selection baseline as an oracle reference.
Due to runtime, MST/PrivBayes are reported at $\varepsilon=1.0$ in figures.

\paragraph{Privacy.}
All mechanisms satisfy $(\varepsilon,\delta)$-DP with $\delta = 1/n^2$ using the Gaussian mechanism and basic sequential composition across measurements.
When PRISM-Predictive performs DP feature selection, we allocate 10\% of the total privacy budget to selection and use the remaining 90\% for synthesis.

\subsection{Evaluation protocol and reporting}
\label{sec:ext-eval}

\paragraph{Train-on-synthetic, test-on-real (TSTR).}
For each method and privacy budget $\varepsilon$, we generate a synthetic training set $\widetilde{D}$ and fit a logistic regression classifier on $\widetilde{D}$ to predict $Y$. We then evaluate ROC-AUC on a held-out real test set from the same benchmark. This matches the intended use case (a downstream analyst trains on released synthetic data) and isolates predictive utility.

\paragraph{Privacy budgets and repetitions.}
For PRISM variants we sweep $\varepsilon \in \{0.2,0.5,1.0,2.0,4.0\}$ on both SCM benchmarks and Adult. We repeat each configuration for 10 random seeds and report mean with 95\% confidence intervals. For computationally heavier external baselines (MST and PrivBayes), we report $\varepsilon=1.0$ due to runtime; this still provides a meaningful point of comparison because it matches a commonly used privacy level in synthesis benchmarks.

\paragraph{Hyperparameters.}
On the SCM benchmarks, PRISM-Causal uses $S=\Pa(Y)$ and PRISM-Graphical uses $S=\MB(Y)$ (ground truth). On Adult, PRISM-Predictive uses DP MI feature selection with $k=8$. For all PRISM methods, we use $n_{\mathrm{syn}}=5000$ unless otherwise noted; we keep the same evaluation model and preprocessing across methods to ensure differences are driven by the synthesizers rather than the downstream learner.

\subsection{Additional tables and diagnostics}

\paragraph{How to read these tables.}
The main paper focuses on trends across privacy budgets (Figures~\ref{fig:scm-shift}, \ref{fig:scm-marginal}, \ref{fig:adult}). Here we provide representative $\varepsilon=1.0$ tables and additional diagnostics to make the wins and trade-offs explicit and to aid reproducibility.

\paragraph{SCM spurious shift table.}
Under spurious shift, the only stable signal is the causal mechanism $P(Y\mid A,B)$. PRISM-Causal preserves this mechanism and remains robust, while approaches that preserve the Markov blanket (including spurious children) or that perform correlation-based selection fail once the children break.
\begin{table}[t]
\centering
\small
\begin{tabular}{ll}
\toprule
Method & AUC (95\% CI) \\
\midrule
PRISM-Causal (Unif) & 0.733 $\pm$ 0.004 \\
PRISM-Causal (Opt) & 0.732 $\pm$ 0.002 \\
MST (Private-PGM) & 0.519 $\pm$ 0.004 \\
PRISM-Graphical (Opt) & 0.513 $\pm$ 0.005 \\
PRISM-Graphical (Unif) & 0.513 $\pm$ 0.005 \\
All Features (Unif) & 0.512 $\pm$ 0.006 \\
Independence Only & 0.502 $\pm$ 0.018 \\
PrivBayes (DataSynthesizer) & 0.501 $\pm$ 0.009 \\
Corr Top-k (Oracle) & 0.490 $\pm$ 0.011 \\
\bottomrule
\end{tabular}

\caption{SCM spurious shift at $\varepsilon=1.0$: TSTR ROC-AUC (mean with 95\% CI).}
\label{tab:scm-shift}
\end{table}
At $\varepsilon=1.0$, PRISM-Causal achieves AUC around $0.73$, while MST is near $0.52$, PrivBayes is near $0.50$, and correlation top-$k$ is near chance. This is the strongest setting for PRISM: it directly matches Regime~1's goal of shift robustness.

\paragraph{SCM marginal shift table.}
Under marginal shift, the spurious children remain predictive and thus blanket-targeting is appropriate for accuracy, while parent-only targeting trades accuracy for robustness. This controlled contrast is the motivation for reporting both PRISM-Causal and PRISM-Graphical in the main text.
\begin{table}[t]
\centering
\small
\begin{tabular}{ll}
\toprule
Method & AUC (95\% CI) \\
\midrule
MST (Private-PGM) & 1.000 $\pm$ 0.000 \\
PRISM-Graphical (Unif) & 1.000 $\pm$ 0.000 \\
PRISM-Graphical (Opt) & 1.000 $\pm$ 0.000 \\
All Features (Unif) & 1.000 $\pm$ 0.000 \\
PrivBayes (DataSynthesizer) & 0.977 $\pm$ 0.013 \\
Corr Top-k (Oracle) & 0.942 $\pm$ 0.010 \\
PRISM-Causal (Unif) & 0.713 $\pm$ 0.027 \\
PRISM-Causal (Opt) & 0.710 $\pm$ 0.037 \\
Independence Only & 0.465 $\pm$ 0.118 \\
\bottomrule
\end{tabular}

\caption{SCM marginal shift at $\varepsilon=1.0$: TSTR ROC-AUC (mean with 95\% CI).}
\label{tab:scm-marginal}
\end{table}
At $\varepsilon=1.0$, PRISM-Graphical reaches near-perfect AUC (matching MST and the all-features baseline), while PRISM-Causal is substantially lower. This is the expected trade-off: including stable children improves accuracy, but would be harmful in the spurious-shift benchmark.

\paragraph{Adult task performance table.}
Adult has moderate dimensionality and a large training set, so many workload-based methods are near the ``All Features'' internal baseline; the primary gap is between task-aware marginals and generic end-to-end baselines (MST/PrivBayes) at the same privacy budget.
\begin{table}[t]
\centering
\small
\begin{tabular}{ll}
\toprule
Method & AUC (95\% CI) \\
\midrule
All Features (Unif) & 0.875 $\pm$ 0.001 \\
PRISM-Predictive (Opt) & 0.874 $\pm$ 0.002 \\
PRISM-Predictive (Unif) & 0.874 $\pm$ 0.003 \\
Corr Top-k (Oracle) & 0.874 $\pm$ 0.002 \\
PrivBayes (DataSynthesizer) & 0.831 $\pm$ 0.016 \\
MST (Private-PGM) & 0.792 $\pm$ 0.014 \\
Independence Only & 0.519 $\pm$ 0.042 \\
\bottomrule
\end{tabular}

\caption{Adult Income at $\varepsilon=1.0$: TSTR ROC-AUC (mean with 95\% CI).}
\label{tab:adult}
\end{table}
At $\varepsilon=1.0$, PRISM-Predictive essentially matches the strongest internal baselines but improves substantially over MST and PrivBayes. This suggests that on large-$n$, moderate-$d$ tabular datasets, the main advantage is avoiding generic structure selection rather than aggressively tuning allocation.

\paragraph{Adult fidelity diagnostics.}
Figure~\ref{fig:adult-fidelity} reports 1-way marginal L1 error. While fidelity does not fully determine predictive utility, it provides a sanity check that methods with extremely poor TSTR performance (e.g., independence-only) also exhibit large marginal distortion.
\begin{figure}[t]
\centering
\includegraphics[width=\linewidth]{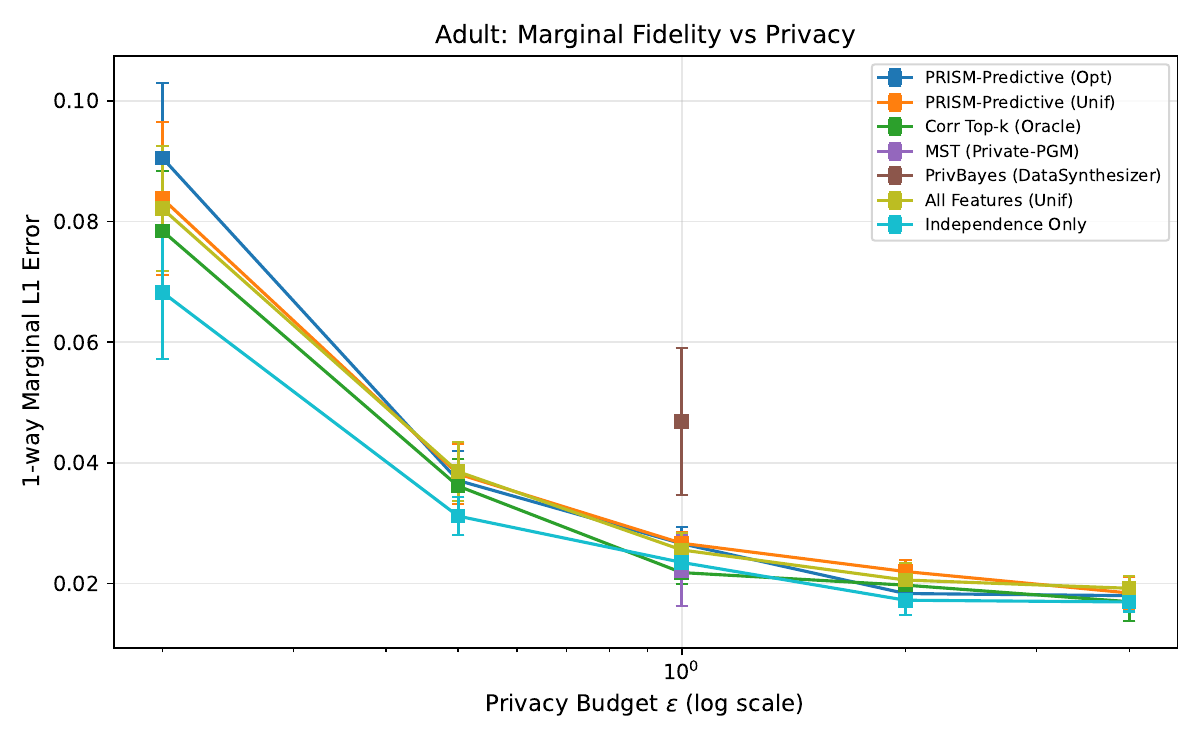}
\caption{Adult Income: 1-way marginal L1 error vs.\ $\varepsilon$ (mean with 95\% CI).}
\label{fig:adult-fidelity}
\end{figure}

\begin{table}[t]
\centering
\small
\begin{tabular}{llll}
\toprule
Method & AUC (95\% CI) & Accuracy & 1-way L1 \\
\midrule
All Features (Unif) & 0.875 $\pm$ 0.001 & 0.800 & 0.026 \\
PRISM-Predictive (Opt) & 0.874 $\pm$ 0.002 & 0.805 & 0.027 \\
PRISM-Predictive (Unif) & 0.874 $\pm$ 0.003 & 0.802 & 0.027 \\
Corr Top-k (Oracle) & 0.874 $\pm$ 0.002 & 0.798 & 0.022 \\
PrivBayes (DataSynthesizer) & 0.831 $\pm$ 0.016 & 0.793 & 0.047 \\
MST (Private-PGM) & 0.792 $\pm$ 0.014 & 0.760 & 0.022 \\
Independence Only & 0.519 $\pm$ 0.042 & 0.761 & 0.023 \\
\bottomrule
\end{tabular}

\caption{Adult Income at $\varepsilon=1.0$: extended metrics.}
\end{table}
The extended metrics help verify that PRISM's gains are not artifacts of the evaluation model: improvements track both predictive utility and (moderate) fidelity, without requiring non-private feature selection.

\subsection{Allocation Wins Benchmark}
\label{sec:alloc-win}

To isolate the effect of closed-form allocation, we construct a synthetic Naive Bayes dataset with heterogeneous feature importance.
We sample $Y\sim\mathrm{Bernoulli}(0.5)$, then sample $d=20$ binary features independently conditioned on $Y$.
Four \emph{strong} features use $P(X=1\mid Y=1)=0.9$ and $P(X=1\mid Y=0)=0.1$; the remaining features are \emph{weak} with $P(X=1\mid Y=1)=0.55$ and $P(X=1\mid Y=0)=0.45$.
We use $n_{\mathrm{train}}=400$ and $n_{\mathrm{test}}=2000$.
Closed-form allocation uses oracle task weights proportional to $(P(X=1\mid Y=1)-P(X=1\mid Y=0))^2$ to reflect heterogeneity; this benchmark is meant to validate the allocation principle, not feature selection.

\paragraph{Interpretation.}
This benchmark is designed so that accurate estimation of the four strong conditionals $P(X_j\mid Y)$ dominates prediction, while the remaining features contribute little. Under tight privacy (small $\varepsilon$) and small $n_{\mathrm{train}}$, uniform allocation wastes budget on weak features and injects unnecessary noise into the strong ones. Closed-form allocation concentrates budget on the strong features, yielding a large AUC gap that is consistent with the square-root weighting in Theorem~\ref{thm:alloc}.

\begin{figure}[t]
\centering
\includegraphics[width=\linewidth]{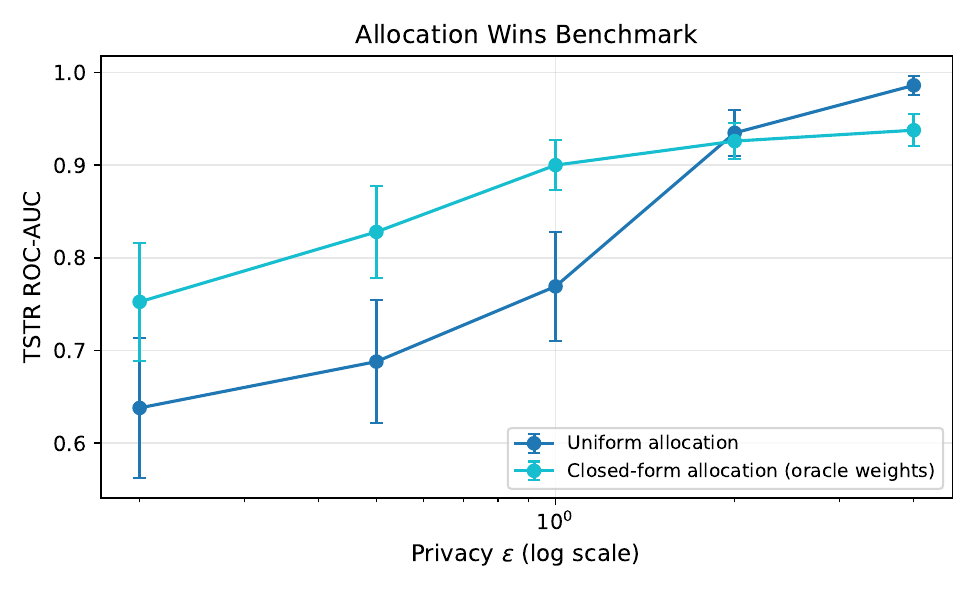}
\caption{Allocation wins: closed-form allocation vs.\ uniform on a heterogeneous-importance Naive Bayes dataset.}
\end{figure}

\begin{table}[t]
\centering
\small
\begin{tabular}{ll}
\toprule
Method & AUC (95\% CI) \\
\midrule
Closed-form allocation (oracle weights) & 0.900 $\pm$ 0.027 \\
Uniform allocation & 0.769 $\pm$ 0.059 \\
\bottomrule
\end{tabular}

\caption{Allocation wins at $\varepsilon=1.0$: TSTR ROC-AUC (mean with 95\% CI).}
\end{table}

\end{document}